\documentclass[11pt]{article}
\usepackage{comment}
\usepackage[margin=1in]{geometry}
\usepackage[T1]{fontenc}
\usepackage{lmodern}
\usepackage{microtype}
\usepackage{amsmath,amssymb,amsthm,mathtools}
\usepackage{authblk}
\usepackage[colorlinks=true,linkcolor=blue,citecolor=blue,urlcolor=blue]{hyperref}

\newtheorem{theorem}{Theorem}
\newtheorem{proposition}{Proposition}
\newtheorem{lemma}{Lemma}
\theoremstyle{remark}

\newcommand{\X}{\mathcal{X}}
\newcommand{\Y}{\mathcal{Y}}
\newcommand{\F}{\mathcal{F}}
\newcommand{\G}{\mathcal{G}}
\newcommand{\E}{\mathbb{E}}
\newcommand{\Prb}{\mathbb{P}}
\newcommand{\one}{\mathbf{1}}
\newcommand{\epsagn}{\varepsilon_{\mathrm{agn}}}
\newcommand{\TV}{\operatorname{TV}}
\newcommand{\KL}{\operatorname{KL}}
\newcommand{\Ber}{\operatorname{Ber}}
\newcommand{\hamming}{d_{\mathrm H}}
\DeclareMathOperator*{\argmax}{arg\,max}

\hypersetup{
  pdftitle={A Rate Separation for Agnostic Direct Sums},
  pdfauthor={Mihir More, Aritra Das, and Debayan Gupta},
  pdfsubject={Agnostic PAC learning curves under direct sums}
}

\title{A Rate Separation for Agnostic Direct Sums}
\author[1]{Mihir More}
\author[1]{Aritra Das}
\author[1]{Debayan Gupta}
\affil[1]{Truth Audit Labs}
\date{}

\begin{document}
\maketitle

\begin{abstract}
Hanneke, Moran, and Waknine \cite{HannekeMoranWaknine2024} asked how the agnostic PAC learning curve of the direct sum $C^r$ depends on the single-instance learning curve $\epsagn(n\mid C)$ and on $r$. We show that the single-instance learning rate does not determine the direct-sum rate. Let $\F$ be the class of the two constant binary functions and let $\G$ consist of the zero function and the identity function. Both classes have agnostic learning curve of order $n^{-1/2}$. 
\end{abstract}

\section{Introduction}

Let $C_1\subseteq \Y_1^{\X_1}$ and $C_2\subseteq \Y_2^{\X_2}$ be concept classes. Following Hanneke, Moran, and Waknine~\cite{HannekeMoranWaknine2024}, their direct sum $C_1\otimes C_2$ is the class of functions
\[
  (c_1\otimes c_2)(x_1,x_2)
  =\bigl(c_1(x_1),c_2(x_2)\bigr),
  \qquad c_i\in C_i.
\]
For a class $C$, its $r$-fold direct sum is denoted by
\[
  C^r=\bigotimes_{i=1}^r C.
\]

The agnostic learning curve measures excess risk relative to the best member of $C^r$. An upper bound on the risk of a product hypothesis does not directly control this excess risk. This is noted in~\cite{HannekeMoranWaknine2024}.
We give a negative answer to the rate-level form of the question,  the order of $\epsagn(n\mid C)$ and the number of factors do not determine the order of $\epsagn(n\mid C^r)$. 
To the best of our knowledge, no subsequent paper explicitly resolves this question in its unrestricted distribution-free form. Suruga~\cite{Suruga2024} proves direct-sum theorems in a different complexity framework, and Holzman, Moran, and Shlimovich~\cite{HolzmanMoranShlimovich2026} study uniform laws of large numbers in product spaces under structural assumptions on the distribution. Recent work on agnostic multiclass learning gives class-specific sample-complexity bounds in terms of combinatorial dimensions~\cite{CohenEtAl2025,Pabbaraju2026}. These results are consistent with the separation proved here but do not express the direct-sum curve as a function of the single-instance learning curve alone.

\section{Preliminaries}\label{sec:preliminaries}

We use the setup and notation of~\cite{HannekeMoranWaknine2024}. Let $\X$ be a domain, let $\Y$ be a label space, and let $k\geq 1$. Write
\[
  \binom{\Y}{k}=\{B\subseteq \Y: |B|=k\}.
\]
A $k$-list function is a map $c:\X\to\binom{\Y}{k}$, and a $k$-list concept class is a set
\[
  C\subseteq \binom{\Y}{k}^{\X}.
\]
A $k$-list learning rule is a map
\[
  A:(\X\times\Y)^*\longrightarrow \binom{\Y}{k}^{\X}.
\]
For a distribution $D$ on $\X\times\Y$, the population loss of a $k$-list function $c$ is
\begin{equation}\label{eq:list-loss}
  L_D(c)
  =\E_{(x,y)\sim D}\bigl[\one\{y\notin c(x)\}\bigr].
\end{equation}

When $k=1$, we identify a singleton $\{y\}$ with its unique element $y$. Under this identification, a $1$-list function is an ordinary function $c:\X\to\Y$, and a $1$-list concept class is an ordinary concept class $C\subseteq\Y^{\X}$. This is precisely the case considered in ~\cite{HannekeMoranWaknine2024}. We therefore set $k=1$ from this point onward. The loss in~\eqref{eq:list-loss} becomes
\[
  L_D(h)=\E_{(x,y)\sim D}\bigl[\one\{h(x)\neq y\}\bigr].
\]
For $C\subseteq\Y^{\X}$, write
\[
  L_D(C)=\inf_{c\in C}L_D(c).
\]
If $A$ is a learning rule and $S\sim D^n$, define
\[
  L_{D,n}(A)=\E_{S\sim D^n}\bigl[L_D(A(S))\bigr].
\]
The agnostic PAC learning curve is
\begin{equation}\label{eq:agnostic-curve}
  \epsagn(n\mid C)
  =\inf_A\sup_D\bigl(L_{D,n}(A)-L_D(C)\bigr),
\end{equation}
where the supremum is over all distributions on $\X\times\Y$. The learning rule is not required to output a member of $C$.

For concept classes $C_i\subseteq\Y_i^{\X_i}$, their direct sum is
\[
  C_1\otimes C_2
  =\{c_1\otimes c_2:c_i\in C_i\},
  \qquad
  (c_1\otimes c_2)(x_1,x_2)
  =\bigl(c_1(x_1),c_2(x_2)\bigr).
\]
For $r\geq 1$, we write $C^r=\bigotimes_{i=1}^r C$. Its domain is $\X^r$, its label space is $\Y^r$, and its loss is zero-one loss on the full vector:
\[
  L_D(h)=\Prb_{(x,y)\sim D}\bigl(h(x)\neq y\bigr).
\]
In particular, this is not coordinatewise Hamming loss.

We write $a_{n,r}\asymp b_{n,r}$ if there are universal constants $c,C>0$ such that
\[
  cb_{n,r}\leq a_{n,r}\leq Cb_{n,r}
\]
for all relevant $n$ and $r$.

We use four standard facts. First, empirical risk minimization over a finite class $H$ satisfies
\begin{equation}\label{eq:finite-class}
  \epsagn(n\mid H)
  \leq K\min\left\{1,\sqrt{\frac{\log |H|}{n}}\right\}
\end{equation}
for a universal constant $K$; see~\cite[Chapter~12]{DevroyeGyorfiLugosi1996}. Second, Le Cam's two-point inequality states that, for distributions $P_+$ and $P_-$ and any test $\widehat\sigma$ taking values in $\{-1,+1\}$,
\begin{equation}\label{eq:le-cam}
  \max_{\sigma\in\{-1,+1\}}
  P_\sigma(\widehat\sigma\neq\sigma)
  \geq \frac{1-\TV(P_+,P_-)}{2};
\end{equation}
see~\cite[Chapter~2]{Tsybakov2009}. Third, Pinsker's inequality gives
\begin{equation}\label{eq:pinsker}
  \TV(P,Q)\leq \sqrt{\frac12\KL(P\|Q)}.
\end{equation}
Finally, we use the following standard form of Assouad's lemma. For $\theta\in\{-1,+1\}^r$, let $\theta^{(j)}$ be obtained from $\theta$ by changing the sign of its $j$th coordinate, and define
\[
  \hamming(a,b)=\sum_{j=1}^r\one\{a_j\neq b_j\}.
\]

\begin{lemma}[Assouad's lemma; \cite{Assouad1983,Tsybakov2009}]\label{lem:assouad}
Let $\{P_\theta:\theta\in\{-1,+1\}^r\}$ be a family of distributions on a common measurable space. If
\[
  \TV(P_\theta,P_{\theta^{(j)}})\leq \eta
\]
for every $\theta$ and $j$, then every estimator $\widehat\theta$ with values in $\{-1,+1\}^r$ satisfies
\[
  \sup_{\theta\in\{-1,+1\}^r}
  \E_\theta\bigl[\hamming(\widehat\theta,\theta)\bigr]
  \geq \frac r2(1-\eta).
\]
\end{lemma}

\section{Main result}\label{sec:main-result}

Set $\X=\Y=\{0,1\}$. Define
\begin{align*}
  \F&=\{f_0,f_1\}, & f_b(x)&=b,\\
  \G&=\{g_0,g_1\}, & g_b(x)&=bx.
\end{align*}
Thus $\F$ consists of the two constant binary functions, while $\G$ consists of the zero function and the identity function.

\begin{theorem}[Rate separation]\label{thm:main}
There are universal constants $0<c<C<\infty$ such that, for all $n,r\geq 1$,
\begin{align}
  \frac{c}{\sqrt n}
  &\leq \epsagn(n\mid \F^r)
  \leq \frac{C}{\sqrt n},
  \label{eq:F-rate}\\[0.3em]
  c\min\left\{1,\sqrt{\frac rn}\right\}
  &\leq \epsagn(n\mid \G^r)
  \leq C\min\left\{1,\sqrt{\frac rn}\right\}.
  \label{eq:G-rate}
\end{align}
Consequently,
\[
  \epsagn(n\mid \F)\asymp
  \epsagn(n\mid \G)\asymp n^{-1/2},
\]
but the rates of their direct sums differ as $r$ grows.
\end{theorem}

The single-instance statement follows by taking $r=1$ in~\eqref{eq:F-rate} and~\eqref{eq:G-rate}. The main assertion is that the two rates cease to agree after taking direct sums. We prove~\eqref{eq:F-rate} and~\eqref{eq:G-rate} in the next two sections.

\subsection{Direct sums of the constant class}

A hypothesis in $\F^r$ is indexed by $u=(u_1,\ldots,u_r)\in\{0,1\}^r$ and predicts $u$ at every input. Write this hypothesis as $h_u$. For a distribution $D$ on $\X^r\times\Y^r$, define
\begin{equation}\label{eq:empirical-distribution}
  p_u=\Prb_D(Y=u),
  \qquad
  \widehat p_u=\frac1n\sum_{i=1}^n\one\{Y_i=u\}.
\end{equation}
Let
\[
  u^\star\in\argmax_{u\in\{0,1\}^r}p_u,
  \qquad
  \widehat u\in\argmax_{u\in\{0,1\}^r}\widehat p_u.
\]
The empirical risk minimizer over $\F^r$ is the constant hypothesis $h_{\widehat u}$.

\begin{proposition}\label{prop:F}
For all $n,r\geq 1$,
\[
  \epsagn(n\mid \F^r)\asymp \frac1{\sqrt n},
\]
where the implicit constants are universal.
\end{proposition}

\begin{proof}
Since $L_D(h_u)=1-p_u$,
\[
  L_D(h_{\widehat u})-L_D(\F^r)
  =p_{u^\star}-p_{\widehat u}.
\]
The definition of $\widehat u$ gives
\begin{align*}
  p_{u^\star}-p_{\widehat u}
  &\leq
  (p_{u^\star}-\widehat p_{u^\star})
  +(\widehat p_{\widehat u}-p_{\widehat u})\\
  &\leq 2\|\widehat p-p\|_\infty
  \leq 2\|\widehat p-p\|_2.
\end{align*}
Consequently, by Jensen's inequality,
\begin{align*}
  \E\bigl[p_{u^\star}-p_{\widehat u}\bigr]
  &\leq 2\sqrt{\E\|\widehat p-p\|_2^2}\\
  &=2\sqrt{\sum_u\operatorname{Var}(\widehat p_u)}\\
  &=2\sqrt{\frac{1-\sum_u p_u^2}{n}}
  \leq \frac{2}{\sqrt n}.
\end{align*}
This proves the upper bound uniformly over $D$ and $r$.

For the lower bound, fix an input $x_0\in\X^r$ and restrict the label distribution to $0^r$ and $1^r$, with probabilities $1/2+\alpha$ and $1/2-\alpha$. Distinguishing which label is more likely is the standard two-point Bernoulli problem. Le Cam's method with $\alpha$ of order $n^{-1/2}$ gives an expected excess risk of order $n^{-1/2}$; see~\cite[Chapter~2]{Tsybakov2009}. 
\end{proof}

\subsection{Direct sums of the zero and identity functions}\label{sec:G}

For $b=(b_1,\ldots,b_r)\in\{0,1\}^r$, write
\[
  h_b=g_{b_1}\otimes\cdots\otimes g_{b_r}\in\G^r.
\]
Because $g_b(x)=bx$,
\begin{equation}\label{eq:product-prediction}
  h_b(x_1,\ldots,x_r)=(b_1x_1,\ldots,b_rx_r).
\end{equation}
Let $e_j$ be the $j$th standard basis vector in $\{0,1\}^r$. Then
\begin{equation}\label{eq:basis-prediction}
  h_b(e_j)=
  \begin{cases}
    0^r,& b_j=0,\\
    e_j,& b_j=1.
  \end{cases}
\end{equation}

The upper bound in~\eqref{eq:G-rate} follows from the finite-class estimate. Since $|\G^r|=2^r$, equation~\eqref{eq:finite-class} gives
\begin{equation}\label{eq:G-upper}
  \epsagn(n\mid\G^r)
  \leq K\min\left\{1,\sqrt{\frac{r\log 2}{n}}\right\}.
\end{equation}

For the lower bound, fix a learning rule $A$ and a number $0<\alpha\leq 1/4$. For every $\theta=(\theta_1,\ldots,\theta_r)\in\{-1,+1\}^r$, define a distribution $D_\theta$ on $\X^r\times\Y^r$ as follows. Choose $J$ uniformly from $\{1,\ldots,r\}$, set $X=e_J$, and, conditional on $J=j$, set
\begin{equation}\label{eq:G-distributions}
  Y=
  \begin{cases}
    e_j,&\text{with probability }\frac12+\theta_j\alpha,\\
    0^r,&\text{with probability }\frac12-\theta_j\alpha.
  \end{cases}
\end{equation}
For a fixed $j$, the more likely label at $e_j$ is $e_j$ when $\theta_j=+1$ and $0^r$ when $\theta_j=-1$. Define
\[
  b_j^\star(\theta)=\frac{1+\theta_j}{2}.
\]
By~\eqref{eq:basis-prediction}, $h_{b^\star(\theta)}$ predicts the more likely label at every $e_j$. Therefore
\begin{equation}\label{eq:G-optimal-risk}
  L_{D_\theta}(\G^r)=\frac12-\alpha.
\end{equation}

Let $S\sim D_\theta^n$. From the output of $A$ define $\widehat\theta_A(S)\in\{-1,+1\}^r$ by
\begin{equation}\label{eq:theta-hat}
  \widehat\theta_{A,j}(S)=
  \begin{cases}
    +1,&A(S)(e_j)=e_j,\\
    -1,&A(S)(e_j)\neq e_j.
  \end{cases}
\end{equation}
We now relate errors in $\widehat\theta_A$ to excess risk. Fix $j$. If $\theta_j=+1$ but $\widehat\theta_{A,j}=-1$, then the learner does not predict the more likely label $e_j$. Predicting $0^r$ has conditional error $1/2+\alpha$, and any other prediction has conditional error one. If $\theta_j=-1$ but $\widehat\theta_{A,j}=+1$, then the learner predicts $e_j$ although $0^r$ is more likely, and its conditional error is $1/2+\alpha$. In either case, the conditional excess over the optimal error $1/2-\alpha$ is at least $2\alpha$. Since $X=e_j$ with probability $1/r$, each incorrect coordinate contributes at least $2\alpha/r$ to the population excess risk. Hence, for every $\theta$ and every sample $S$,
\begin{equation}\label{eq:risk-hamming}
  L_{D_\theta}(A(S))-L_{D_\theta}(\G^r)
  \geq \frac{2\alpha}{r}\,
  \hamming\bigl(\widehat\theta_A(S),\theta\bigr).
\end{equation}
This inequality also covers improper learning rules, since predictions outside $\{0^r,e_j\}$ have conditional error one.

It remains to bound the accuracy with which $\theta$ can be estimated. Let $\theta^{(j)}$ be obtained by changing only the sign of $\theta_j$. The distributions $D_\theta$ and $D_{\theta^{(j)}}$ agree unless $X=e_j$, an event of probability $1/r$. Conditional on this event, the probability of the label $e_j$ changes from $1/2+\alpha$ to $1/2-\alpha$, or conversely. Therefore
\begin{align}
  \KL(D_\theta\|D_{\theta^{(j)}})
  &=\frac1r\,
  \KL\left(
    \Ber\left(\frac12+\alpha\right)
    \middle\|
    \Ber\left(\frac12-\alpha\right)
  \right)\notag\\
  &=\frac{2\alpha}{r}
  \log\left(\frac{1+2\alpha}{1-2\alpha}\right)
  \leq \frac{16\alpha^2}{r}.
  \label{eq:adjacent-kl}
\end{align}
By independence,
\begin{equation}\label{eq:product-kl}
  \KL(D_\theta^n\|D_{\theta^{(j)}}^n)
  \leq \frac{16n\alpha^2}{r}.
\end{equation}
Choose
\begin{equation}\label{eq:alpha}
  \alpha=\frac1{16}
  \min\left\{1,\sqrt{\frac rn}\right\}.
\end{equation}
Then the right-hand side of~\eqref{eq:product-kl} is at most $1/16$. Pinsker's inequality gives
\begin{equation}\label{eq:adjacent-tv}
  \TV(D_\theta^n,D_{\theta^{(j)}}^n)
  \leq \sqrt{\frac1{32}}<\frac14
\end{equation}
for every $\theta$ and $j$. Applying Assouad's lemma to the family $\{D_\theta^n\}$,
\begin{equation}\label{eq:hamming-lower}
  \sup_{\theta\in\{-1,+1\}^r}
  \E_{S\sim D_\theta^n}
  \bigl[\hamming(\widehat\theta_A(S),\theta)\bigr]
  \geq \frac{3r}{8}.
\end{equation}
Combining~\eqref{eq:risk-hamming} and~\eqref{eq:hamming-lower},
\begin{align*}
  \sup_D\bigl(L_{D,n}(A)-L_D(\G^r)\bigr)
  &\geq
  \sup_{\theta\in\{-1,+1\}^r}
  \E_{S\sim D_\theta^n}
  \bigl[L_{D_\theta}(A(S))-L_{D_\theta}(\G^r)\bigr]\\
  &\geq \frac{2\alpha}{r}\cdot\frac{3r}{8}
  =\frac{3\alpha}{4}.
\end{align*}
Since $A$ was arbitrary, taking the infimum over $A$ and substituting~\eqref{eq:alpha} proves the lower bound in~\eqref{eq:G-rate}. Together with~\eqref{eq:G-upper}, this proves~\eqref{eq:G-rate} and completes the proof of Theorem~\ref{thm:main}.

\section{Conclusion}

The classes $\F$ and $\G$ have the same single-instance agnostic learning rate, but their direct sums do not. When $1\leq r\leq n$, the rate for $\G^r$ is larger than the rate for $\F^r$ by a factor of order $\sqrt r$. When $r\geq n$, the minimax excess risk for $\G^r$ is bounded below by a positive constant, while the rate for $\F^r$ remains of order $n^{-1/2}$. Therefore the asymptotic order of $\epsagn(n\mid C)$ and the value of $r$ are not sufficient to determine the asymptotic order of $\epsagn(n\mid C^r)$ for an arbitrary concept class $C$.

\bibliographystyle{alpha}
\bibliography{references}

\end{document}